\documentclass{llncs}
\usepackage{times}
\usepackage{mathptmx}
\usepackage{latexsym}
\usepackage{amsmath}
\usepackage{amssymb}
\usepackage{mathrsfs}
\usepackage{color,graphicx,subfig}
\usepackage{epic}
\usepackage{fancybox,graphpap}
\usepackage[all,arc,curve,color,frame,tips]{xy}
\usepackage{pb-diagram,pb-xy}
\usepackage{bbm}
\usepackage{array}

\begin{document}

\newcommand{\END}{\qed}
\newcommand{\LT}{\mathcal{L}}
\newcommand{\eqa}{\thickapprox}
\newcommand{\eqb}{\equiv}
\newcommand{\PS}[1]{\mathscr{P}#1}
\newcommand{\PSB}[1]{\widehat{\mathscr{P}}#1}
\newcommand{\h}[1]{\overline{#1}}
\newcommand{\set}[1]{\{#1\}}
\newcommand{\Al}{\biguplus}
\newcommand{\RC}{\div_R}
\newcommand{\LC}{\div_L}
\newcommand{\wei}{\textit{weight}}
\newcommand{\LG}[1]{\langle #1 \rangle}
\newcommand{\E}{\mathbb{S}}
\newcommand{\EC}{\widehat{\mathcal{S}}}
\newcommand{\PO}{\prec}
\newcommand{\st}{\text{\textit{\#step}}}
\newcommand{\com}{<\!\!>}
\newcommand{\df}{\ \stackrel{\mbox{\textit{\scriptsize{df}}}}{=}\ }
\newcommand{\iffdf}{\stackrel{\mbox{\textit{\scriptsize{df}}}}{\iff}\ }
\newcommand{\calf}[1]{\mathcal{#1}}
\newcommand{\sq}{\sqsubset}
\newcommand{\ccl}{\;\bowtie}
\newcommand{\todo}[1]{ \textcolor{red}{TODO: #1}}
\newcommand{\tcomment}[1]{\text{\hspace*{2mm}$\langle$~\parbox[t]{\textwidth}{ #1 $\rangle$}}}
\newcommand{\sym}[1]{{#1}^{\leftrightarrows\;}}
\newcommand{\lex}[1]{\,{#1}^{\textit{lex}}\,}
\newcommand{\stor}[1]{\,{#1}^{\textit{st}}\,}
\newcommand{\reco}[1]{ {#1}^{\,\textbf{\textit{C}}} }
\newcommand{\It}[1]{\mathit{#1}}

\newcommand{\seq}[1]{\la#1\ra}
\newcommand{\defref}[1]{Definition~\ref{def:#1}}
\newcommand{\theoref}[1]{Theorem~\ref{theo:#1}}
\newcommand{\propref}[1]{Proposition~\ref{prop:#1}}
\newcommand{\lemref}[1]{Lemma~\ref{lem:#1}}
\newcommand{\exref}[1]{Example~\ref{ex:#1}}
\newcommand{\reref}[1]{Remark~\ref{re:#1}}
\newcommand{\eref}[1]{\eqref{eq:#1}}
\newcommand{\figref}[1]{Fig.~\ref{fig:#1}}

\newcommand{\secref}[1]{Section~\ref{sec:#1}}
\newcommand{\eqnref}[1]{Eq.~(\ref{eq:#1})}
\newcommand{\colq}[1]{\left(\begin{array}{c}#1\end{array}\right)}
\newcommand{\colqq}[1]{\left(\begin{array}{ll}#1\end{array}\right)}
\newcommand{\cmodel}[1]{\mathfrak{#1}}
\newcommand{\model}[1]{\mathcal{#1}}
\newcommand{\Tuniverse}{\mathcal{T}_\It{univ}}
\newcommand{\Tspec}{\mathcal{T}_\It{spec}}
\newcommand{\Tgso}{\mathcal{T}_\It{gso}}
\newcommand{\Tpslcore}{\mathcal{T}_\It{pslcore}}
\newcommand{\Cuniverse}{\cmodel{M}^\It{univ}}
\newcommand{\Cspec}{\cmodel{M}^\It{spec}}
\newcommand{\Cgso}{\cmodel{M}^\It{gso}}
\newcommand{\CO}[1]{\mathbb{C}(#1)}
\newcommand{\ICO}[1]{\mathbb{IC}(#1)}
\newcommand{\EOD}{{\hfill \begin{scriptsize}$\blacksquare$\end{scriptsize}}}
\newcommand{\tot}{\blacktriangleleft}
\newcommand{\G}[1]{\mathbb{G}(#1)}
\newcommand{\GI}[1]{\widehat{\mathbb{G}}(#1)}

\numberwithin{equation}{section}

\title{Modelling Concurrent Behaviors in the Process Specification Language}
\author{Dai Tri Man L\^e}
\institute{Department of Computer Science,\\
University of Toronto,\\
Toronto, ON, M5S 3G4 Canada\\
{ledt@cs.toronto.edu}}

\date{}

\maketitle
\pagestyle{headings} 
\begin{abstract}
In this paper, we propose a first-order ontology for generalized stratified order structure. We then classify the models of the theory using model-theoretic techniques. An ontology mapping from this ontology to the core theory of Process Specification Language is also discussed.
\end{abstract}

\section{Introduction}
In Process Specification Language (PSL), the ordering of event (activity) occurrences is modelled using occurrence trees, which are restricted forms of partial orders. Although partial orders can sufficiently model the ``earlier than" relationship, they cannot \emph{explicitly} model the ``not later than" relationship \cite{J4}.  For instance, if an event $a$ is performed ``not later than" an event $b$, then this ``not later than" relationship can be modelled by the following set of two step sequences $x=\{\{a\}\{b\},\{a,b\}\}$, where the \textit{step} $\{a,b\}$ models the simultaneous performance of $a$ and $b$. But the set $x$ can not be represented by any partial order.

To provide a unified framework for analyzing ``earlier than'' and ``not later than'' relationships, we proposed to interpret the \emph{generalized stratified order structure} (\emph{gso-structure}) theory  within PSL. The gso-structure theory is originated from causal partial order theory and \emph{stratified order structure} (\emph{so-structure}) theory. A so-structure \cite{GP,JK0,J1,JK} is a triple $(X,\prec,\sqsubset)$, where $\prec$ and $\sqsubset$ are binary relations on $X$. They were invented to model both ``earlier than" (the relation $\prec$) and ``not later than" (the relation $\sqsubset$) relationships, under the assumption that all \emph{system runs} (also called \emph{observations}) are modelled by \emph{stratified orders}, i.e., step sequences. They have been successfully applied to model inhibitor and priority systems, asynchronous races, synthesis problems, etc. (see for example \cite{J1,JLM,KK} and others).  However, so-structures can adequately model concurrent histories only when the paradigm $\pi_3$ of \cite{J4,JK} is satisfied. Paradigm $\pi_3$ says that if two event occurrences are observed in both orders of execution, then they will also be observed executing simultaneously. Without this assumption, we need gso-structures, which were introduced and analyzed in \cite{GJ}. The comprehensive theory for gso-structures has been developed in \cite{J0,Le}. A gso-structure is a triple $\left(X,\com,\sqsubset\right)$, where $\com$ and $\sqsubset$ are binary relations on $X$ modelling ``never simultaneously'' and ``not later than'' relationships respectively under the assumption that all system runs are modelled by stratified orders.  Intuitively, gso-structures can model even the situation when we have the mixture of ``true concurrency'' and interleaving semantics. The only disadvantage is that gso-structures are more complex to conceptualize than so-structures. 

Since the works of Janicki et al. \cite{J4,J0} focus on the algebraic properties of gso-structures, the number of axioms are kept to minimal and some of the assumptions are made implicit. Furthermore, the theorems of gso-structure theory frequently involve quantifying over relations, which requires the use of higher-order language. Hence, to apply first-order ontology and model-theoretic techniques in the manner as in \cite{Gr2}, we will first define a formal ontology for gso-structure in first-order logic and characterize all possible models of gso-structure theory up to isomorphism. After that we can proceed to investigate to what extend the theorems of gso-structure theory hold within the first-order setting of PSL by studying possible ontological mappings from gso-structure theory to PSL. 

The organization of this paper is as follows. In Section~2, we will give a first-order axiomatization of the gso-structure theory and end the section will a result showing that our theory is consistent. In Section~3, we will classify all possible models of the gso-structure theory from Section~2 using more natural and intuitive concepts from graph theory. In Section~4, we study a semantic mapping from our theory to PSL-core theory. Section~5 contains our concluding remarks.

\newcommand{\event}{\mathsf{event}}
\newcommand{\occurrence}{\mathsf{occurrence}}
\newcommand{\eventocc}{\mathsf{event\_occurrence}}
\newcommand{\observation}{\mathsf{observation}}
\newcommand{\obefore}{\mathsf{observed\_before}}
\newcommand{\osim}{\mathsf{observed\_simult}}
\newcommand{\notlaterthan}{\mathsf{not\_later\_than}}
\newcommand{\earlierthan}{\mathsf{earlier\_than}}
\newcommand{\commutativity}{\mathsf{nonsimultaneous}}

\section{First-order axiomatization of gso-structure theory}
The following table provides a summary of the lexicon of so-structure theory. The relations $\PO$, $\sq$ and $\com$ in the papers of Janicki et al. \cite{J0,J4} correspond to the relations $\earlierthan$, $\notlaterthan$ and $\commutativity$ respectively in this paper. We rename these relations to make the theory more intuitive and accessible.
\begin{center}
\begin{tabular}{|p{2.5cm}||p{4cm}|p{5cm}|}
\hline
&  Lexicon  & Informal Semantics \\
\hline\hline
Universe &	$\event(e)$  & $e$ is an event \\
\cline{2-3}
	&$\eventocc(o)$  & $o$ is an event occurrence\\
\cline{2-3}
	&$\observation(x)$  & $x$ is an observation \\
\cline{2-3}
	&$\occurrence(o,e)$  & $o$ is an event occurrence of event $e$ \\
\hline\hline
Gso-structure&$\earlierthan(o_1,o_2)$  & $o_1$ must occur earlier than $o_2$ \\
\cline{2-3}
	&$\notlaterthan(o_1,o_2)$  & $o_1$ must occur not later than $o_2$ \\
\cline{2-3}
	&$\commutativity(o_1,o_2)$  & $o_1$ and $o_2$ must not occur simultaneously \\
\hline\hline
Observations&$\obefore(o_1,o_2,x)$  & event occurrence $o_1$ is observed before event occurrence $o_2$ in observation $x$ \\
\cline{2-3}
	&$\osim(o_1,o_2,x)$  & event occurrences $o_1$ and  $o_2$ are observed simultaneously in observation $x$ \\
\hline
\end{tabular} 
\end{center}

\subsection{Events, event occurrences and observations \label{sec:events}}
Everything is either an \emph{event}, \emph{event occurrence} or \emph{observation}:
\begin{align}
(\forall x)\left( \event(x)\vee\eventocc(x)\vee\observation(x)\right)
\label{eq:e1}
\end{align}
The sets of events, event occurrences and observations are pair-wise disjoint.
\begin{align}
(\forall x)
\colq{\neg(\event(x)\wedge\eventocc(x))\wedge\neg(\event(x)\wedge\observation(x))\\
\wedge\neg(\eventocc(x)\wedge\observation(x))}
\label{eq:e2}
\end{align}
The \emph{occurrence} relation only holds between events and event occurrences.
\begin{align}
(\forall e,o)\left(\occurrence(o,e)\supset \event(e)\wedge\eventocc(o)\right)
\label{eq:e3}
\end{align}
Every event occurrence is an occurrence of some event.
\begin{align}
(\forall o)\left(\eventocc(o)\supset (\exists e)\left(\event(e)\wedge\occurrence(o,e)\right)\right)
\label{eq:e4}
\end{align}
Every event occurrence is an occurrence of a unique event.
\begin{align}
(\forall o_1,e_1,e_2)\left(\occurrence(o_1,e_1)\wedge\occurrence(o_1,e_2)\supset e_1=e_2\right)
\label{eq:e5}
\end{align}

\subsection{Gso-structure and its relations \label{sec:gso}}
\begin{sloppypar}
We now axiomatize the gso-structure, which describes the \emph{specification level} of a concurrent system. The relations of gso-structure are $\earlierthan$, $\notlaterthan$ and $\commutativity$. The relation $\earlierthan$ can be defined as the intersection of the latter two, yet is added because it helps to make our axioms shorter and more intuitive.\\

We have to make sure that the field of the relations $\earlierthan$, $\notlaterthan$ and $\commutativity$  consists of only event occurrences.
\end{sloppypar}
\begin{align}
(\forall o_1,o_2)\colq{\earlierthan(o_1,o_2) \supset (\occurrence(o_1)\wedge\occurrence(o_2)) } \label{eq:gso1}\\
(\forall o_1,o_2)\colq{\notlaterthan(o_1,o_2)\supset (\occurrence(o_1)\wedge\occurrence(o_2)) }  \label{eq:gso2}\\
(\forall o_1,o_2)\colq{\commutativity(o_1,o_2) \supset (\occurrence(o_1)\wedge\occurrence(o_2)) }  \label{eq:gso3}
\end{align}
The relation $\commutativity$ is irreflexive and symmetric.
\begin{align}
(\forall o_1)(\neg \commutativity(o_1,o_1)) \label{eq:gso4}\\
(\forall o_1,o_2)(\commutativity(o_1,o_2)\supset\commutativity(o_2,o_1)) \label{eq:gso5}
\end{align}
The $\earlierthan$ relation is the intersection of the $\notlaterthan$ and the $\commutativity$ relations.
\begin{align}
(\forall o_1,o_2)\colq{\left(\notlaterthan(o_1,o_2)\wedge\commutativity(o_1,o_2)\right)\\\equiv\earlierthan(o_1,o_2)} \label{eq:gso6}
\end{align}
The $\notlaterthan$ relation is irreflexive.
\begin{align}
(\forall o_1)(\neg \notlaterthan(o_1,o_1)) \label{eq:gso7}
\end{align}
The $\notlaterthan$ and $\earlierthan$ relations satisfy some weak form of transitivity.
\begin{align}
(\forall o_1,o_2,o_3)\colq{\notlaterthan(o_1,o_2)\wedge\notlaterthan(o_2,o_3)\wedge o_1\not=o_3\\\supset\notlaterthan(o_1,o_3)} \label{eq:gso8}\\
(\forall o_1,o_2,o_3)\colq{\colqq{&(\notlaterthan(o_1,o_2)\wedge \earlierthan(o_2,o_3))\\\vee& (\earlierthan(o_1,o_2)\wedge\notlaterthan(o_2,o_3))}\\\supset\earlierthan(o_1,o_3)} \label{eq:gso9}
\end{align}
\bigskip\\

The following propositions are helpful in understanding the relations of a gso-structure. The first proposition basically says that the $\earlierthan$ relation is a partial order.

\begin{proposition}
\begin{align*}
(\forall o_1)(\neg \earlierthan(o_1,o_1)) \\
(\forall o_1,o_2,o_3)\colq{\earlierthan(o_1,o_2)\wedge\earlierthan(o_2,o_3)\\\supset\earlierthan(o_1,o_3)} 
\end{align*}
\label{prop:gso1}
\end{proposition}
\begin{proof} The irreflexivity property follows from axioms \eref{gso6} and \eref{gso7}. The transitivity property follows from axioms \eref{gso6} and \eref{gso9}.\qed 
\end{proof}

The second proposition shows the intuition that if two event occurrences must happen not later than each other, then they must occur simultaneously.

\begin{proposition}
\begin{align*}
(\forall o_1,o_2)\colq{\notlaterthan(o_1,o_2)\wedge\notlaterthan(o_2,o_1)\\\supset\neg\commutativity(o_1,o_2)} 
\end{align*}
\label{prop:gso2}
\end{proposition}
\begin{proof}
We assume for a contradiction that there are some observations $o_1$ and $o_2$ such that
\[\notlaterthan(o_1,o_2)\wedge\notlaterthan(o_2,o_1)\wedge\commutativity(o_1,o_2) \] 
Then since $\notlaterthan(o_1,o_2)\wedge\commutativity(o_1,o_2)$, it follows from axiom \eref{gso6} that $\earlierthan(o_1,o_2)$. But since $\commutativity$ is symmetric (axiom \eref{gso5}), we also have $\notlaterthan(o_2,o_1)\wedge\commutativity(o_2,o_1)$. 

Thus we have  $\earlierthan(o_1,o_2)$ and $\earlierthan(o_2,o_1)$, which by \propref{gso1} implies $\earlierthan(o_1,o_1)$. But this contradicts with \propref{gso1}, which says that the $\earlierthan$ relation is irreflexive. \qed
\end{proof}

The third proposition shows the intuition that if the first event happens earlier than the second event, then it is not the case that the second event happens not later than the first event.
\begin{proposition}
\begin{align*}
(\forall o_1,o_2)\colq{\earlierthan(o_1,o_2)\supset\neg\notlaterthan(o_2,o_1)} 
\end{align*}
\label{prop:gso3}
\end{proposition}
\begin{proof}
We assume for a contradiction that there are some observations $o_1$ and $o_2$ such that
\[(\forall o_1,o_2)\colq{\earlierthan(o_1,o_2)\supset\neg\notlaterthan(o_2,o_1)}  \]
Then by the axiom \eref{gso9}, we have $\earlierthan(o_2,o_1)$. Thus,  $\earlierthan(o_1,o_2)$ and $\earlierthan(o_2,o_1)$, which by \propref{gso1} implies $\earlierthan(o_1,o_1)$. But this contradicts with \propref{gso1}, which says that the $\earlierthan$ relation is irreflexive. \qed
\end{proof}

\begin{example}
Assume the set of all possible event occurrences is $\set{o_i : 1\le i\le 7}$. The following is an example of a gso-structure, where 
\begin{enumerate}
 \item The $\earlierthan$ relation is represented by a directed \emph{acyclic} graph $G_1$:
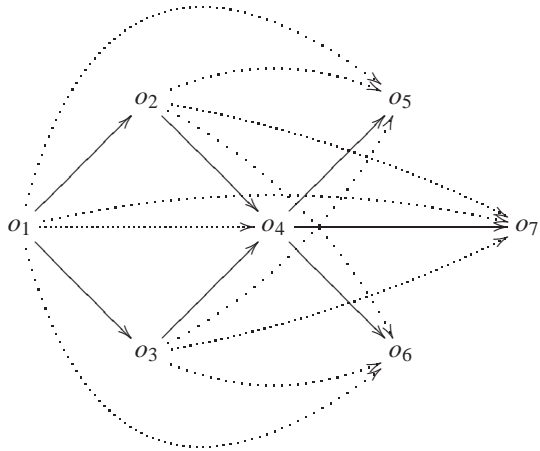
\[
\xymatrix@!=2pc{
			&		&			&		&\\
			&o_2\ar[dr]\ar@{.>}@/^1pc/[rr]\ar@{.>}@/^/[drrr]\ar@{.>}@/^1pc/[ddrr]& 			&o_5	&\\
o_1\ar[dr]\ar[ur]\ar@{.>}[rr]\ar@{.>}@/_5pc/[drrr]\ar@{.>}@/^5pc/[urrr]\ar@{.>}@/^1pc/[rrrr]	
			&		&o_4\ar[dr]\ar[ur]\ar[rr]&	&o_7\\
			&o_3\ar[ur]\ar@{.>}@/_1pc/[rr]\ar@{.>}@/_/[urrr]\ar@{.>}@/_1pc/[uurr]&			&o_6	&\\
			&		&			&		&
}
\]
Note that in this diagram, we used the solid edges to denote the edges of the \emph{transitive reduction}\footnote{A \emph{transitive reduction} of a binary relation $R$ on a set $X$ is a minimal relation $R'$ on $X$ such that the transitive closure of $R'$ is the same as the transitive closure of $R$.} of the $\earlierthan$ relation. 
 \item The $\notlaterthan$ relation is represented as the following directed graph $G_2$:
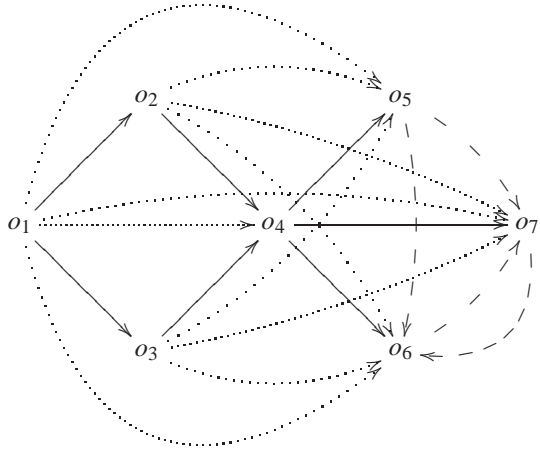
\[
\xymatrix@!=2pc{
			&		&			&		&\\
			&o_2\ar[dr]\ar@{.>}@/^1pc/[rr]\ar@{.>}@/^/[drrr]\ar@{.>}@/^1pc/[ddrr]& 			
										&o_5\ar@/^/@{-->}[dr]\ar@/^/@{-->}[dd]	&\\
o_1\ar[dr]\ar[ur]\ar@{.>}[rr]\ar@{.>}@/_5pc/[drrr]\ar@{.>}@/^5pc/[urrr]\ar@{.>}@/^1pc/[rrrr]		
			&		&o_4\ar[dr]\ar[ur]\ar[rr]&		&o_7\ar@/^2pc/@{-->}[dl]\\
			&o_3\ar[ur]\ar@{.>}@/_1pc/[rr]\ar@{.>}@/_/[urrr]\ar@{.>}@/_1pc/[uurr]&	&o_6\ar@/_/@{-->}[ur]	&\\
			&		&			&		&
}
\]

Note that we used the dashed edges to denote the edges of $G_2$ which are not in $G_1$. \\
 \item The $\commutativity$ relation is represented by the following (undirected) graph $G_3$ (because $\commutativity$ is symmetric). 

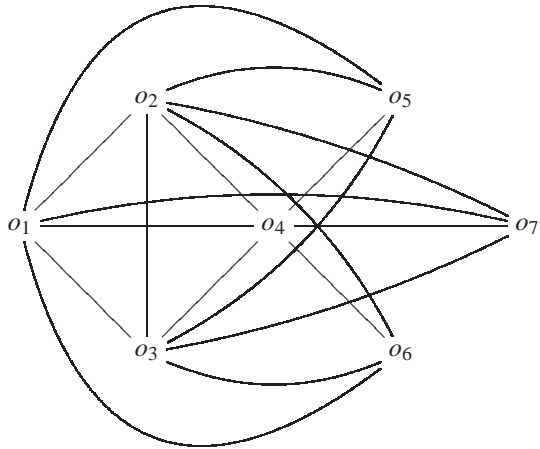
\[
\xymatrix@!=2pc{
			&		&			&		&\\
			&o_2\ar@{-}[dr]\ar@{-}[dd]\ar@{-}@/^1pc/[rr]\ar@{-}@/^/[drrr]\ar@{-}@/^1pc/[ddrr]& 			&o_5	&\\
o_1\ar@{-}[dr]\ar@{-}[ur]\ar@{-}[rr]\ar@{-}@/_5pc/[drrr]\ar@{-}@/^5pc/[urrr]\ar@{-}@/^1pc/[rrrr]	
			&		&o_4\ar@{-}[dr]\ar@{-}[ur]\ar@{-}[rr]&	&o_7\\
			&o_3\ar@{-}[ur]\ar@{-}@/_1pc/[rr]\ar@{-}@/_/[urrr]\ar@{-}@/_1pc/[uurr]&			&o_6	&\\
			&		&			&		&
}
\]

Note that except the edge $\set{o_2,o_3}$, all other edges of $G_3$ are exactly the edges of the \emph{comparability graph} of the  $\earlierthan$ relation. Because of the quantity of edges the comparability graph has, it is often more practical to draw the \emph{complement graph} of the graph induced by the relation $\commutativity$. For example, the complement graph $\bar{G}_3$ of the graph $G_3$ is the following:
\[
\xymatrix@!=1pc{
			&o_2		& 			&o_5\ar@{-}[dr]\ar@{-}[dd]	&\\
o_1			&		&o_4			&				&o_7\\
			&o_3		&			&o_6\ar@{-}[ur]			&
}
\]
\EOD
\end{enumerate}
\label{ex:gso1}
\end{example}

\subsection{Observations and the $\obefore$, $\osim$ relations \label{sec:observation}}
If the relations of a gso-structure in the previous section describe the specification level (also called structural semantics) of a concurrent system, observations characterize \emph{behavioral level} of the system. The $\obefore$ (or $\osim$) relation relates two event occurrences and an observation.
\begin{align}
(\forall o_1,o_2,o)\colq{\obefore(o_1,o_2,o)\\\supset \colq{\eventocc(o_1)\wedge\eventocc(o_1)\\\wedge \observation(o)}} \label{eq:o1}\\
(\forall o_1,o_2,o)\colq{\osim(o_1,o_2,o)\\\supset \colq{\eventocc(o_1)\wedge\eventocc(o_1)\\\wedge \observation(o)}}
\label{eq:o2}
\end{align}

Each observation and the $\obefore$ relation specify a stratified order on the event occurrences as follows. Every event occurrence cannot be observed before itself with respect to any observation.
\begin{align}
(\forall o_1,o)\neg\obefore(o_1,o_1,o) \label{eq:o3}
\end{align}
The $\obefore$ is transitive with respect to any observation.
\begin{align}
(\forall o_1,o_2,o_3,o)\colq{\obefore(o_1,o_2,o)\wedge\obefore(o_2,o_3,o)\\\supset\obefore(o_1,o_3,o)} \label{eq:o4}
\end{align}
The $\osim$ relation and $\obefore$ can be derived from each other.
\begin{align}
(\forall o_1,o_2,o)\colq{\colq{\neg\obefore(o_1,o_2,o)\\\wedge\neg\obefore(o_2,o_1,o)\wedge o_1\not= o_2}\\\equiv\osim(o_1,o_2,o)} \label{eq:o5}
\end{align}
The  $\obefore$ relation on a fixed observation satisfies the stratified order property.
\begin{align}
(\forall o_1,o_2,o_3,o)\colq{\colq{\osim(o_1,o_2,o)\wedge \osim(o_2,o_3,o)}\\\supset(\osim(o_1,o_3,o)\vee o_1=o_3)}
\label{eq:o6}
\end{align}

Every observation and  the $\obefore$ relation specify a \emph{stratified order extension} of the gso-structure. 
\begin{align}
(\forall o_1,o_2)\colq{\commutativity(o_1,o_2)\\\supset(\forall o)(\obefore(o_1,o_2,o)\vee\obefore(o_2,o_1,o))} \label{eq:o7}\\
(\forall o_1,o_2)\colq{\notlaterthan(o_1,o_2)\\\supset(\forall o)\colqq{&\obefore(o_1,o_2,o)\\\vee&\osim(o_1,o_2,o)}} \label{eq:o8}
\end{align}
Axioms \eref{o7} and \eref{o8} impose the \emph{observation soundness} property of our gso-structure theory in the following sense: if $o$ is an possible observation of the system, then it must satisfy the constraints specified by the relations of the gso-structure.

We next axiomatize the \emph{observation completeness} property of our gso-structure theory. If $o_1$ and $o_2$ are simultaneous event occurrences, then there must be some observation $o$, where $o_1$ and $o_2$ are observed simultaneously. 
\begin{align}
(\forall o_1,o_2)\colq{\neg\commutativity(o_1,o_2)\supset(\exists o)\osim(o_1,o_2,o)} \label{eq:o9}
\end{align}
And if it is not the case that the event occurrence $o_1$ is not later than the event occurrence $o_2$, then there will be some observation $o$, where $o_2$ is observed earlier than $o_1$.
\begin{align}
(\forall o_1,o_2)\colq{\neg \notlaterthan(o_1,o_2)\supset(\exists o)\obefore(o_2,o_1,o)} \label{eq:o10}
\end{align}
\smallskip\\

The reason why stratified orders are used to encode observations can be explained formally in the next two propositions.\\

For any observation $o$, we define:
\begin{align*}
\lhd_o&\df\left\{(o_1,o_2):\obefore(o_1,o_2,o)\right\},\\
\frown_o&\df\left\{(o_1,o_2):\osim(o_1,o_2,o)\right\},\\
\simeq_o&\df\left\{(o_1,o_2):\colqq{&\eventocc(o_1)\wedge\eventocc(o_2)\\\wedge&(\osim(o_1,o_2,o)\vee o_1=o_2)}\right\}.
\end{align*}

\begin{proposition}
For all event occurrences $o_1$, $o_2$ and $o_3$, we have
\begin{enumerate}
 \item $o_1\simeq_o o_2$
\item $o_1\simeq_o o_2\supset o_2\simeq_o o_1$
 \item $o_1\simeq_o o_2 \wedge o_2\simeq_o o_3 \supset o_1\simeq_o o_3$
\end{enumerate}
In other words, the relation $\simeq_o$ is an equivalence relation. \label{prop:strat1}
\end{proposition}
\begin{proof} 
\begin{enumerate}
\item Follows from how $\simeq_o$ is defined.
\item Follows from axiom $\eref{o5}$ and how $\simeq_o$ is defined.
\item Follows from axiom $\eref{o6}$ and how $\simeq_o$ is defined. \qed
\end{enumerate}
\end{proof}

The intuition of \propref{strat1} is that for any fixed observation $o$, we can extend the $\osim$ relation with the identity relation to construct the equivalence relation $\simeq_o$. The relation $\simeq_o$ can then be used to partition the set of event occurrences, where we can think of each equivalence class as a ``composite event occurrence'' consisting of only atomic event occurrences that are pairwise observed \emph{simultaneously} within $o$. For example, \figref{observation} shows a stratified order $\lhd_o$ induced by an observation $o$ and the $\obefore$ relation. In this case, the equivalence classes of $\simeq_o$ are the sets $\set{o_1,o_2}$, $\set{o_3}$, $\set{o_4,o_5,o_6}$, $\set{o_7,o_8}$ and $\set{o_9,o_{10}}$, where the fact that $o_1$ and $o_2$ belong to the same equivalence class means they are observed simultaneously within $o$.

\begin{figure}
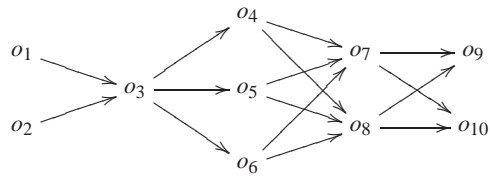

\[
\xygraph{
!{<0cm,0cm>;<1.5cm,0cm>:<0cm,1cm>::}
!~-{@{-}@[|(2)]}
!{(0,0.5)}*+{o_1}="o1" !{(0,-0.5)}*+{o_2}="o2"
!{(1,0)}*+{o_3}="o3"
!{(2,1)}*+{o_4}="o4" !{(2,0)}*+{o_5}="o5" !{(2,-1)}*+{o_6}="o6" 
!{(3,0.5)}*+{o_7}="o7" !{(3,-0.5)}*+{o_8}="o8"
!{(4,0.5)}*+{o_9}="o9" !{(4,-0.5)}*+{o_{10}}="o10"
"o1":"o3" "o2":"o3"
"o3":"o4" "o3":"o5" "o3":"o6"
"o4":"o7" "o5":"o7" "o6":"o7"
"o4":"o8" "o5":"o8" "o6":"o8"
"o7":"o9" "o7":"o10" 
"o8":"o9" "o8":"o10"
} \]
\label{fig:observation} 
\caption{A example of a stratified order $\lhd_o$ induced by an observation $o$ and the $\obefore$ relation. (Edges resulted from transitivity are omitted in this diagram for simplicity.)}
\end{figure}

\begin{proposition} If $A$ and $B$ are two distinct equivalence classes of $\simeq_o$, then either $A\times B \subseteq \lhd_o$ or $B\times A \subseteq \lhd_o$. \label{prop:strat2}
\end{proposition}
\begin{proof}We pick $a\in A$ and $b\in B$. Clearly, $a\lhd_o b$ or $b\lhd_o a$, otherwise $a \frown_o b$ which contradicts that $a$,$b$ are elements from two distinct equivalence classes. There are two cases:
\begin{enumerate}
 \item If $a\lhd_o b$: we want to show $A\times B \subseteq \lhd_o$. Let $c\in A$ and  $d\in B$, it suffices to show $c\lhd_o d$. Assume for contradiction that $\neg (c\lhd_o d)$. Since $c\not\simeq_o d$, it follows that $d\lhd_o c$. There are three different subcases:
	\begin{enumerate}
 	\item If $a=c$, then $d\lhd_o a$ and $a \lhd_o b$. Hence, $d \lhd_o b$. This contradicts that $d,b\in B$.
	\item If $b=d$, then $b\lhd_o c$ and $a \lhd_o b$. Hence, $a \lhd_o c$. This contradicts that $a,c\in A$.
	\item If $a\not=c$ and $b\not=d$, then $a\frown_o c$ and $b\frown_o d$ and $\neg(a\frown_o d)$ and $\neg(c\frown_o b)$. Since $\neg(a\frown_o d)$, either $a\lhd_o d$ or $d\lhd_o a$. 
		\begin{itemize}
		\item If $a\lhd_o d$: since $d\lhd_o c$, it follows $a\lhd_o c$. This contradicts $a\frown_o c$.
		\item If $d\lhd_o a$: since $a\lhd_o b$, it follows $d\lhd_o b$. This contradicts $d\frown_o b$.
		\end{itemize}
	\end{enumerate}
 Therefore, we conclude $A\times B \subseteq \lhd_o$.
 \item If $b\lhd_o a$: using a symmetric argument, it follows that $B\times A \subseteq \lhd_o$.\qed 
\end{enumerate}
\end{proof}

\propref{strat2} leads to the following consequence. For any observation $o$, let us define the relation $\widehat{\lhd}_o$ on the set $E_o=\set{[a]_{\simeq_o}:\eventocc(a)}$ as
\begin{align*}
\widehat{\lhd}_o\df\set{(A,B):A\times B\subseteq \lhd_o \;\wedge\; A\in E_o \;\wedge\; B\in E_o}.
\end{align*}
Then the relation $\widehat{\lhd}_o$ is a \emph{strict total order} on $E_o$. Intuitively, the equivalence classes in $E_o$ can always be totally ordered using $\widehat{\lhd}_o$, where for any two equivalence classes $A$ and $B$ in $E_o$, if $A\;\widehat{\lhd}_o\; B$, then all event occurrences in $A$ are observed before all the event occurrences in $B$ within the observation $o$. 

For examples, the equivalence classes of the stratified order from \figref{observation} can be totally ordered by the ordering $\widehat{\lhd}_o$ as follows: 
\[\set{o_1,o_2}\; \widehat{\lhd}_o\; \set{o_3} \;\widehat{\lhd}_o\; \set{o_4,o_5,o_6}\;\widehat{\lhd}_o\;\set{o_7,o_8}\;\widehat{\lhd}_o\;\set{o_9,o_{10}}\]

When the cardinality of the set of event occurrences is finite as in our example, the stratified order from \figref{observation} can be equivalently represented more compactly as \[\set{o_1,o_2}\set{o_3}\set{o_4,o_5,o_6}\set{o_7,o_8}\set{o_9,o_{10}},\]
where each equivalence class is called a \emph{step} and the whole sequence is called a \emph{step sequence}. 

It might seem counterintuitive that our axioms allow observations whose infinitely many event occurrences are observed simultaneously. However, this is just a limitation of first order theory. Since our theory allows models that observe arbitrarily large finite set of simultaneous event occurrences, by the compactness theorem there will be models whose observations will allow us to observe infinite set of simultaneous event occurrences.\\

\subsubsection*{Observation soundness}
We have just discussed the idea behind why stratified orders are used to formalize the notion of an observation. We next want to show the intuition of how stratified order based observations satisfy the observation soundness properties with respect to a gso-structure. We will do so using a detailed example.

\begin{example}
\begin{sloppypar}
Given the set of event occurrences $\set{o_i\mid 1\le i \le 7}$ and the relations $\earlierthan$, $\notlaterthan$ and $\commutativity$ from \exref{gso1}, we want to know possible observations of this gso-structure. By axioms \eref{o7} and \eref{o8} for observation soundness, we know that all of the observations must satisfy all the causality constraints specified by these three relations. For each observation $\It{ob}$, we let $G_{\It{ob}}$ denote the dag representing the stratified order $\lhd_{\It{ob}}$. 
\begin{enumerate}
 \item The observation $\It{ob}$ satisfies the $\notlaterthan$ relation intuitively meaning that $G_{\It{ob}}$ must contain $G_1$, i.e., $G_1\subseteq G_{\It{ob}}$. 
 \item The observation $\It{ob}$ satisfies the $\notlaterthan$ relation \emph{roughly} which means that $G_{\It{ob}}$ might or might not contains the edges of $G_2-G_1$, where  $G_2-G_1$ denotes the graph difference of $G_2$ and $G_1$. The exception is when $G_2-G_1$ contains both directed edges $(o_i,o_j)$ and $(o_j,o_i)$, then neither $(o_i,o_j)$ nor $(o_j,o_i)$ is allowed to be included in $G_{\It{ob}}$.
 \item Finally $\It{ob}$ satisfies the $\commutativity$ relation is equivalent to saying that if 
 $\set{o_i,o_j}\in G_3$, but neither $(o_i,o_j)$ nor $(o_j,o_i)$ is in  the graph $G_1$, then we have the case that either $(o_i,o_j)$ or $(o_j,o_i)$ must be included in $G_{\It{ob}}$.
\end{enumerate}

From these intuitions, if $\earlierthan$, $\notlaterthan$ and $\commutativity$ are given an interpretation as in \exref{gso1}, then we notice the follows.
\begin{itemize}
 \item Since $(o_3,o_2)\in G_3$ and $(o_2,o_3),(o_3,o_2)\not\in G_1$, if we consider only the set of event occurrences $\set{o_1,o_2,o_3,o_4}$, then the transitive reduction graphs of all of the possible ways they can be observed are:
\[
\xymatrix{
o_1\ar[r]	&o_2\ar[r]	&o_3\ar[r]	&o_4
}
\]

\[
\xymatrix{
o_1\ar[r]	&o_3\ar[r]	&o_2\ar[r]	&o_4
}
\]

\item Since $(o_5,o_6),(o_5,o_7),(o_7,o_6),(o_6,o_7)\in G_2-G_1$, if we consider only the set of event occurrences $\set{o_4,o_5,o_6,o_7}$, then the transitive reduction graphs of all of the possible ways they can be observed are:
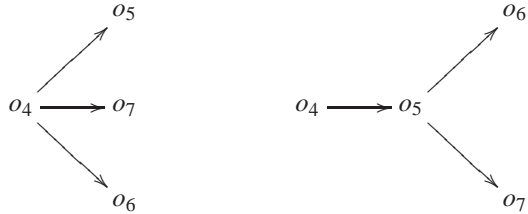
\[
\xymatrix{
 			&o_5			&&		& 		&o_6\\
o_4\ar[dr]\ar[ur]\ar[r]&o_7			&&o_4\ar[r]	&o_5\ar[dr]\ar[ur]	&\\
			&o_6			&&		&		&o_7
}
\]
Note that because $(o_7,o_6),(o_6,o_7)\in G_2-G_1$, the vertices $o_6$ are $o_7$ disconnected (incomparable) in all of the possible observations.
\end{itemize}

Combining all of these cases together, the transitive reduction graphs of all possible observations which satisfy the observation soundness condition with respect to the gso-structure from \exref{gso1} are depicted in \figref{stratext1}.
\begin{figure}[!h]
\begin{enumerate}
 \item[(a)] 
\[
\xymatrix{
		&		&		& 			&o_5			&\\
o_1\ar[r]	&o_2\ar[r]	&o_3\ar[r]	&o_4\ar[dr]\ar[ur]\ar[r]&o_7			&\\
		&		&		&			&o_6			&
}\]
 \item[(b)]
\[
\xymatrix{
		&		&		& 			&o_5			&\\
o_1\ar[r]	&o_3\ar[r]	&o_2\ar[r]	&o_4\ar[dr]\ar[ur]\ar[r]&o_7			&\\
		&		&		&			&o_6			&
}
\]
 \item[(c)]
\[
\xymatrix{
		&		&		&		& 		&o_6		\\
o_1\ar[r]	&o_2\ar[r]	&o_3\ar[r]	&o_4\ar[r]	&o_5\ar[dr]\ar[ur]	&			\\
		&		&		&		&		&o_7		
}
\]
 \item[(d)] 
\[
\xymatrix{
		&		&		&		& 		&o_6		\\
o_1\ar[r]	&o_3\ar[r]	&o_2\ar[r]	&o_4\ar[r]	&o_5\ar[dr]\ar[ur]	&			\\
		&		&		&		&		&o_7		
}
\]
\end{enumerate}
\caption{Transitive reduction graphs of all possible observations which satisfy the observation soundness condition with respect to the gso-structure from \exref{gso1}. \label{fig:stratext1}}
\end{figure}
\EOD
\end{sloppypar} 
\label{ex:stratext1}
\end{example}

\subsubsection*{Observation completeness}
One subtle question one might ask is if the observation completeness condition is too strong for every gso-structure to have. In other words, is there any model of our theory, where its gso-structure cannot be characterized by any set of stratified order observations? Fortunately, the theorem which we will discuss next will help us answer this question. Before stating the theorem, let us define some notations.

For a partial order $\lhd$ on a set $X$, let us define 
\begin{align*}
\lhd^{\frown}&\df\set{(x,y)\in X\times X: x\not= y \wedge \neg y\lhd x}\\
\sym{\lhd} &\df \lhd \cup \lhd^{-1}
\end{align*}

The following theorem can be seen as a generalization of Szpilrain's theorem \cite{Szp}. If Szpilrajn's Theorem ensures that every partial order can be uniquely reconstructed from the set of all of its total order extensions, then the following theorem states that every gso-structure can be uniquely reconstructed from its stratified order extensions.

\begin{theorem}[Guo and Janicki \cite{GJ}] Let \[\model{M}=(X,\earlierthan^{\model{M}},\notlaterthan^{\model{M}},\commutativity^{\model{M}})\] be a gso-structure, i.e., $\model{M}$ satisfies all axioms from \eref{gso4} to \eref{gso9}. Let $\Omega$ be the set of all stratified orders $\lhd$ on $X$ satisfying the following stratified order extension conditions:
\begin{enumerate}
\item $\commutativity^{\model{M}} \subseteq \sym{\lhd}$ and
\item $\notlaterthan^{\model{M}} \subseteq \lhd^\frown$.
\end{enumerate} 
Then we have $\commutativity^{\model{M}} = \bigcap_{\lhd\in \Omega}\sym{\lhd}$ and $\notlaterthan^{\model{M}} =\bigcap_{\lhd\in \Omega} \lhd^\frown$. \qed
\label{theo:gsoext}
\end{theorem}

From this theorem, we know that there is always a subset of $\Omega$, where we can uniquely reconstruct $\notlaterthan^{\model{M}}$  and $\commutativity^{\model{M}}$. Note that although the consequence of the theorem does not mention $\earlierthan^{\model{M}}$, the axiom \eref{gso6} implies that
\[\earlierthan^{\model{M}}=\notlaterthan^{\model{M}}\cap\commutativity^{\model{M}}.\]
Thus, $\earlierthan^{\model{M}}=\Bigl(\bigcap_{\lhd\in \Omega}\sym{\lhd}\Bigr)\cap\Bigl(\bigcap_{\lhd\in \Omega} \lhd^\frown\Bigr)= \bigcap_{\lhd\in \Omega}\lhd.$ Hence, observation completeness is a safe assumption for our gso-structure theory. 

It is worth noticing that, since \theoref{gsoext} is a generalization of Szpilrajn's Theorem, the proof of \theoref{gsoext} requires \emph{the axiom of choice}.

\begin{example}
Let $\lhd_a$, $\lhd_b$, $\lhd_c$ and $\lhd_d$ be the stratified orders whose transitive reduction graphs are depicted in cases (a), (b), (c) and (d) respectively. Then the set of all the stratified order extensions of the gso-structure from \exref{gso1} is $\Omega=\set{\lhd_a,\lhd_b,\lhd_c,\lhd_d}$. However, the gso-structure from \exref{gso1} can be uniquely reconstructed from any subset of $\Omega$, which is a superset of at least  one of the following two sets $\set{\lhd_a,\lhd_d}$ and $\set{\lhd_b,\lhd_c}$. 

For example, let us consider the set $\set{\lhd_a,\lhd_d}$. Then the relations $\lhd_a^\frown$ and $\lhd_d^\frown$ can be represented as the following two graphs (some arcs which can be inferred from transitivity are omitted for simplicity):
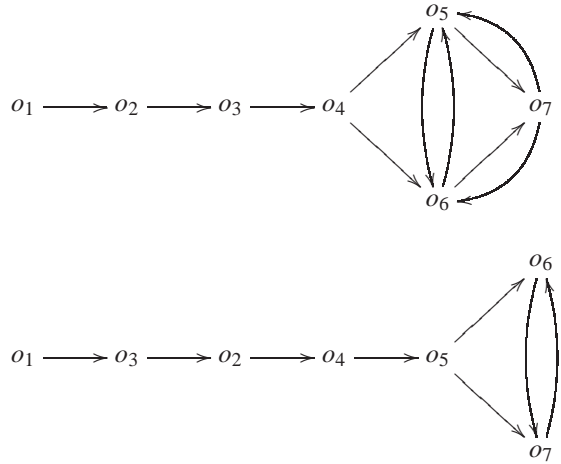
\[
\xymatrix{
		&		&		& 			&o_5\ar@/_/[dd]\ar[dr]&\\
o_1\ar[r]	&o_2\ar[r]	&o_3\ar[r]	&o_4\ar[dr]\ar[ur]	&			&o_7\ar@/^1pc/[dl]\ar@/_1pc/[ul]\\
		&		&		&			&o_6\ar@/_/[uu]\ar[ur]			&
}\]
\[
\xymatrix{
		&		&		&		& 		&o_6\ar@/_/[dd]\\
o_1\ar[r]	&o_3\ar[r]	&o_2\ar[r]	&o_4\ar[r]	&o_5\ar[dr]\ar[ur]	&			\\
		&		&		&		&		&o_7\ar@/_/[uu]		
}
\]
It is easy to check that the graph $G_2$ is exactly the intersection of these two graphs. It is also easy to check that the graph $G_3$ is the intersection of the comparability graphs induced by the relations $\lhd_a$ and $\lhd_d$. \EOD
\label{ex:stratext2}
\end{example}

Let $\Tgso$ denote our gso-structure theory, which consists of axioms from \eref{e1} to \eref{o10}. Then we have the following theorem.

\begin{theorem}
The theory $\Tgso$ is consistent. 
\label{theo:consistent}
\end{theorem}
\begin{proof}
It suffices to build a model $\model{M}$ that satisfies all of these axioms.  Let $E$, $\mathit{EO}$ and $O$ be three \emph{pairwise disjoint} sets, where 
\begin{align*}
E &= \set{e_i : 1\le i\le 7}\\
\It{EO} &= \set{o_i : 1\le i\le 7}\\
O &= \set{\It{ob}_a,\It{ob}_d} 
\end{align*}

We define the universe of $\model{M}$ to be the set $U\df E\cup \mathit{EO} \cup O$. We then give the following interpretations
\begin{enumerate}
 \item $\event^{\model{M}}=E$
 \item $\eventocc^{\model{M}}=\mathit{EO}$
 \item $\observation^{\model{M}}=O$
 \item $\occurrence^{\model{M}}=\set{(e_i,o_i): 1\le i\le 7}$
 \item $\earlierthan^{\model{M}}$ is exactly the graph $G_1$ from \exref{gso1}
 \item $\notlaterthan^{\model{M}}$ is exactly the graph $G_2$ from \exref{gso1}
 \item $\commutativity^{\model{M}}$ is exactly the graph $G_3$ from \exref{gso1}
 \item $\obefore^{\model{M}}=\set{(o_{1},o_{2},ob_a): o_{1}\lhd_a o_{2}}\cup \set{(o_{1},o_{2},ob_d): o_{1}\lhd_d o_{2}}$, where $\lhd_a$ and $\lhd_d$ are relations from \exref{stratext2}.
 \item  $\osim^{\model{M}}=\set{(o_{1},o_{2},ob_a): o_{1}\frown_a o_{2}}\cup \set{(o_{1},o_{2},ob_d): o_{1}\frown_d o_{2}}$, where $\frown_a$ is the following relation
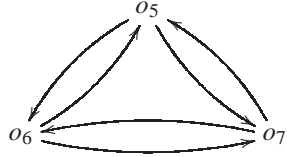
\[
\xymatrix@!=2pc{
			&o_5\ar@/_/[dl]\ar@/_/[dr]	&\\
o_6\ar@/_/[rr]\ar@/_/[ur]	&			&o_7\ar@/_/[ll]\ar@/_/[ul]
}\]
and $\frown_d$ is the following relation
\[
\xymatrix@!=2pc{
o_6\ar@/_/[rr]&&o_7\ar@/_/[ll]
}\]
\end{enumerate}
It is easy to check that axioms \eref{e1} to \eref{e5} are satisfied by this interpretation. We also see from \exref{gso1} how the interpretation of $\earlierthan$, $\notlaterthan$ and $\commutativity$ given by $G_1$, $G_2$ and $G_3$ respectively satisfies that axioms from \eref{gso1} to \eref{gso9}. It is also clear from \exref{stratext1} and \exref{stratext2} that our interpretation satisfies axioms from \eref{o1} to \eref{o10}.
\qed
\end{proof}

\section{Models of the theory $\Tgso$}
By \theoref{consistent}, we already know that $\Tgso$ is consistent, and hence the class of all models satisfying $\Tgso$ is nonempty. In this section, we will attempt to classify all the possible models of our theory $\Tgso$. For convenience, we let $\Tuniverse$ denote the theory consisting of axioms from \eref{e1} to \eref{e5}, and we let $\Tspec$ denote the specification-level theory consisting of axioms from \eref{gso1} to \eref{gso9}.

\subsection{Events and their occurrences}
The following definition will give us the classification of all models of $\Tuniverse$.
\begin{definition}
Let $\Cuniverse$ denote the class of all possible models for $\Tuniverse$. Then any model $\model{M}\in\Cuniverse$ consists of the following sets $E$, $\It{EO}$, and $O$ such that
\begin{enumerate}
 \item the universe $M$ of $\model{M}$ is $E\cup\It{EO}\cup O$
 \item $E$, $\It{EO}$, and $O$ are pairwise disjoint 
 \item $E$ is a partitioning of the set $\It{EO}$.
 \item $E=\event^\model{M}$
 \item $\It{EO}=\eventocc^\model{M}$
 \item $O=\observation^\model{M}$
 \item $\occurrence^\model{M}=\set{(x,e)\in EO\times E:x\in e}$
\end{enumerate}\EOD 
\label{def:class1}
\end{definition}

The correctness of our definition follows from the following theorem.

\begin{theorem}[Satisfiability Theorem for $\Tuniverse$] If the class $\Cuniverse$ is defined as in \defref{class1}, then for any model $\model{M}\in\Cuniverse$, we have $\model{M}\models\Tuniverse$.
\label{theo:sat1}
\end{theorem}
\begin{proof}
The fact that $\model{M}$ satisfies axioms \eref{e1} and \eref{e2} follows from the condition that $E$, $\It{EO}$, and $O$ are pairwise disjoint. The fact that $\model{M}$ satisfies axioms \eref{e3} and \eref{e5} follows from our construction that $E$ is a partitioning of the set $\It{EO}$ and the interpretation of $\occurrence$ as the membership relation between $EO$ and $E$.
\qed 
\end{proof}

\begin{theorem}[Axiomatizability Theorem for $\Tuniverse$] Any model  of $\Tuniverse$ is isomorphic to a structure of $\Cuniverse$.
\label{theo:axiom1} 
\end{theorem}
\begin{proof}
Let $\model{M}$ be a model of $\Tuniverse$. We will show that $\model{M}$ satisfies the conditions of the structures in $\Cuniverse$ from \defref{class1}. 

Since $\model{M}\models \eref{e1}$ , we know that any element of the universe $M$ of $\model{M}$ belongs to one of the following sets $\event^\model{M}$, $\eventocc^\model{M}$ and $\observation^\model{M}$. Since $\model{M}\models \eref{e2}$,  all of these sets $\event^\model{M}$, $\eventocc^\model{M}$ and $\observation^\model{M}$ are pairwise disjoint. Hence, the conditions (1), (2), (4)--(6) are satisfied.

Since $\model{M}$ satisfied axioms \eref{e3}--\eref{e4}, we know that $\occurrence^\model{M}$ a function \[\occurrence^\model{M}: \eventocc^\model{M}\rightarrow\event^\model{M} \]
Hence, given the set $EO=\eventocc^\model{M}$, we can define the set $E$ as 
\[E\df\set{\set{x:f(x)=e}: e\in \event^\model{M}}\]
Since $\occurrence^\model{M}$ is a function, it can be easily checked that $E$ defines a partitioning of $EO$. Thus, the condition (3) and (7) are also satisfied.
\qed
\end{proof}

\subsection{Graph-theoretic classification of gso-structures}
We will classify the relational models of $\Tspec$ in a more well-understood combinatorial setting.  But before that we will recall some definitions. 

\begin{definition}
A directed graph $G$ is a pair $(V,E)$, where $V$ is the set of vertices and $E\subseteq V\times V \setminus \set{(v,v)\in V\times V}$ is the set of edges. 
\begin{itemize}
 \item The transitive closure of $G$ is a graph $G^+ = (V,E^+)$ such that for all $v,w$ in $V$ there is an edge $(v,w)$ in $E^+$ if and only if there is a nonempty path from $v$ to $w$ in $G$.
 \item The graph $G$ is called a \emph{transitive graph} if we have $E=E^+\setminus\set{(v,v)\in E^+}$. In other words, $G$ is its own transitive-closure taken away all the self-loops.
 \item We let $\CO{G}=(V,\CO{E})$ denote the \emph{comparability graph} of $G$, i.e., $$\CO{E}=\set{(u,v): (u,v)\in E}.$$
 \item We let $\ICO{G}=(V,\ICO{E})$ denote the \emph{incomparability graph} of $G$, i.e., $$\ICO{E}=\set{(u,v): (u,v)\not\in E}.$$
 \item We  let $\overline{G}=(V,\overline{E})$ denote the \emph{complement graph} of $G$, i.e., $$\overline{E}=\set{(u,v)\in V\times V:u\not=v\;\wedge\;(u,v)\not\in E}.$$ 
 In other words, we exclude the self-loops.
  \item Given a directed graph $H=(V,E')$, we write $G\subseteq H$ if $E\subseteq E'$. We write $G-H$ to denote the graph $(V,E\setminus E')$. And we write $G\cup H$ to denote the graph $(V,E\cup E')$.
\end{itemize} \EOD
\end{definition}

In this paper, we will treat undirected graphs (or graphs) as a special case of directed graph, where the edge relations are symmetric. This explains why we defined $\CO{G}$ and $\ICO{G}$ as direct graphs. Also note that whenever we call something a graph or a directed graph, we already mean that it does not contain any self-loop. \\

\begin{definition}
Let $\Cspec$ denote the class of all possible models for $\Tspec$. Then any model $\model{M}\in\Cuniverse$ can be uniquely determined from the following three graphs:
\begin{enumerate}
 \item The graph $G_1=(\It{EO},E_1)$ is a \emph{acyclic} transitive graph.
 \item The graph $G_2=(\It{EO},E_2)$ is a transitive graph satisfying the following two conditions:
	\begin{enumerate}
	 \item $G_2=G_1\cup G_3$, where $G_3\subseteq \ICO{G_1}$.
	 \item the graph $G_2$ does not contain a triangle that has any of these two forms:
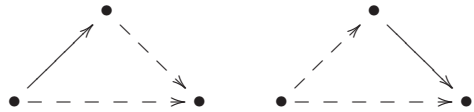
\[
\xymatrix{
			&\bullet\ar@{-->}[dr]	&\\
\bullet\ar@{-->}[rr]\ar[ur]	&			&\bullet
}\mbox{\quad\quad}\xymatrix{
			&\bullet\ar[dr]	&\\
\bullet\ar@{-->}[rr]\ar@{-->}[ur]	&			&\bullet
}\]
where the solid edges are edges of $G_1$ and the dashed edges are edges of $G_3$.
	\end{enumerate}
 \item The graph $G_4=(\It{EO},E_4)$ is an undirected graph such that there is an undirected graph $G_5\subseteq \ICO{G_2}$ and $G_4=\CO{G_1}\cup G_5$.
\end{enumerate}
The interpretation for $\model{M}$ can be defined as:
\begin{itemize}
 \item the universe $M$ of $\model{M}$ is a superset of $\It{EO}$
 \item $\eventocc^\model{M} = \It{EO}$
 \item $\earlierthan^\model{M} = E_1$
\item $\notlaterthan^\model{M} = E_2$
 \item $\commutativity^\model{M} = E_4$.
\end{itemize}
\EOD
\label{def:class2} 
\end{definition}

\begin{theorem}[Satisfiability Theorem for $\Tspec$] If the class $\Cspec$ is defined as in \defref{class2}, then for any model $\model{M}\in\Cspec$, we have $\model{M}\models\Tspec$.
\label{theo:sat2}
\end{theorem}
\begin{proof} Since $\earlierthan^\model{M}$, $\notlaterthan^\model{M}$ and $\commutativity^\model{M}$ are exactly the edge relations of $G_1$,  $G_2$ and $G_3$ respectively, it follows that $\model{M}$ satisfies axioms \eref{gso1}--\eref{gso3}. 

Since $\commutativity^\model{M} = E_4$ and  $G_4$ is a graph, it follows that $\model{M}$ satisfies axioms \eref{gso4} and \eref{gso5}.
 
Recall that we define $\earlierthan^\model{M} = E_1$ and $\notlaterthan^\model{M} = E_2$. Hence, to show that $\model{M}\models\eref{gso6}$, it suffices to show the following lemma.
\begin{itemize}
 \item[] 
\begin{lemma}
 $G_1=G_2\cap G_4$. 
\label{lem:l1}
\end{lemma}
\begin{proof}[Proof of \lemref{l1}]
($\subseteq$) From \defref{class2}, we know that  $G_4=\CO{G_1}\cup G_5$ and $G_2=G_1\cup G_3$. Hence, it follows that $\CO{G_1}\cap G_1\subseteq G_4\cap G_1$. But we know that  $G_1=\CO{G_1}\cap G_1$.\\

($\supseteq$) It suffices to show that $G_5\cap G_2 = \emptyset$ and $G_3\cap G_4 = \emptyset$. But we know that $G_5\cap G_2 = \emptyset$ since from condition (3) of \defref{class2}, we have $G_5\subseteq \ICO{G_2}$. This also implies that that $G_3\cap G_5=\emptyset$. It remains to show that $G_3\cap\CO{G_1}=\emptyset$, but this holds since from condition (2)(a) of \defref{class2} we have $G_3\subseteq \ICO{G_1}$.
\qed
\end{proof}
\end{itemize}

Since $G_2$ is a transitive graph, it follows that $\model{M}$ satisfies axioms \eref{gso7} and \eref{gso8}. It remains to show that $\model{M}\models\eref{gso9}$. Then since $G_2=G_1\cup G_3$, there are three cases to consider:
\begin{itemize}
 \item If $(o_1,o_2)\in E_1$ and  $(o_2,o_3)\in E_1$, then it follows that $(o_1,o_3)\in E_1$ since $G_1$ is a transitive graph.
 \item If $(o_1,o_2)\in E_3$ and $(o_2,o_3)\in E_1$, where $E_3$ is the set of edges of $G_3$, then since $G_2$ is a transitive graph, we know that $(o_1,o_3)\in E_2$. Suppose for a contradiction that $(o_1,o_3)\in E_3$, then we have a triangle
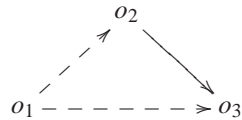
\[\xymatrix{
			&o_2\ar[dr]	&\\
o_1\ar@{-->}[rr]\ar@{-->}[ur]	&			&o_3
}\]
This is a contradiction.
 \item The case of $(o_1,o_2)\in E_1$ and $(o_2,o_3)\in E_3$ is similar to the previous case.
\end{itemize}
\qed
\end{proof}

\begin{theorem}[Axiomatizability Theorem for $\Tspec$] Any model  of $\Tspec$ is isomorphic to a structure of $\Cspec$.
\label{theo:axiom2} 
\end{theorem}
\begin{proof}
Let $\model{M}$ be a model of $\Tspec$. We will show that $\model{M}$ satisfies the conditions of the structures in $\Cspec$ from \defref{class2}. 

Since $\model{M}$ satisfies axioms \eref{gso1} and \eref{gso2}, we know that we can determine the vertex set $EO=\eventocc^\model{M}$ for the graphs $G_1$, $G_2$ and $G_3$.

Since $\model{M}$ satisfies all axioms, from \propref{gso1} we know that $\earlierthan^\model{M}$ is a strict partial order, so it can be represented by an acyclic transitive graph $G_1$ as from the condition~{(1)} of \defref{class2}.

Since $\model{M}$ satisfies axioms \eref{gso7} and \eref{gso8}, we can represent the $\notlaterthan^\model{M}$ relation by a transitive graph $G_2$ as from the condition~{(2)} of \defref{class2}. 
\begin{itemize}
 \item To show that the condition (2)(a) is satisfied, we must show that $G_3=G_2-G_1 \subseteq \ICO{G_1}$. Suppose for a contradiction that there is an edge $(u,v)$ that appears on both  $G_3$ and $\ICO{G_1}$. Since $G_3=G_2-G_1$, we know that $(u,v)\not\in E_1$, so $(v,u)\in E_1$. This would mean that $\earlierthan^\model{M}(v,u)$ and $\notlaterthan^\model{M}(u,v)$. But this contradicts with \propref{gso3}. 
 \item To show that the condition (2)(b) is satisfied, we assume for a contradiction that we have at least one of the following two triangles:
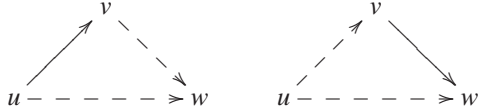
\[\xymatrix{
			&v\ar@{-->}[dr]	&\\
u\ar@{-->}[rr]\ar[ur]	&			&w
}\mbox{\quad\quad}\xymatrix{
			&v\ar[dr]	&\\
u\ar@{-->}[rr]\ar@{-->}[ur]	&			&w
}\]
where the solid edges are edges of $G_1$ and the dashed edges are edges of $G_3$. The left triangle implies that $\earlierthan^\model{M}(u,v)$ and $\notlaterthan^\model{M}(v,w)$ but $\notlaterthan^\model{M}(u,w)$. This contradicts with axiom \eref{gso9}. Similarly the case of the right triangle also leads to a contradiction.
\end{itemize}

Since $\model{M}$ satisfies axioms \eref{gso4} and \eref{gso5}, we can represent $\commutativity^\model{M}$ by a graph $G_4$ as from the condition~{(3)} of \defref{class2}. Let $G_5=G_4-\CO{G_1}$, it remains to show that $G_5\subseteq \ICO{G_2}$. Suppose for a contradiction that an edge $(u,v)$ and $(v,u)$ is shared by both the graph $G_5$ and $\CO{G_2}$. Without loss of generality, we can assume that $(u,v)\in G_2$. Thus, $\notlaterthan^\model{M}(u,v)$ and $\commutativity^\model{M}(u,v)$. But by axiom \eref{gso6}, we have that $\earlierthan^\model{M}(u,v)$. This contradicts with our assumption that $(u,v)\in G_5=G_4-\CO{G_1}$. \qed
\end{proof}

\subsection{Observations}
We first introduce a more combinatorial representation of stratified orders.
\begin{definition}
Given a set $X$, we call the pair $(P,\tot)$ a \emph{ranking structure} of $X$ if $P$ is a partitioning of the set $X$ and $\tot$ is a total ordering on the set $P$. \EOD
\end{definition}
Intuitively, a ranking structure of $X$ is just a partitioning $P$ of $X$ equipped with a total ordering which orders the partitions in $P$.

\begin{proposition}
Any stratified order $\lhd$ on a set $X$ can be uniquely determined by a \emph{ranking structure} of $X$.
\end{proposition}
\begin{proof}
Similarly to the ideas from \propref{strat1} and \propref{strat2}, we define an equivalence relation from the stratified order $\lhd$ as follows:
\[\simeq_\lhd\df \set{(x,y)\in X\times X : \neg x\lhd y \wedge \neg y\lhd x \wedge x\not=y}\]
Then let $P$ be the set of all partitions of $X$ with respect to this equivalence relation $\simeq_\lhd$. 

Next we define the relation $\tot$ as $\tot\df\set{(A,B)\in P \times P:A\times B\subseteq \lhd}$. Then, similarly to \propref{strat2}, we can check that $\tot$ is a total ordering.\\

To recover the stratified order $\lhd$ from the ranking structure $(P,\tot)$, we simply reconstruct \[\lhd=\bigcup\set{X\times Y : X\not=Y \wedge X \tot Y}.\]
\qed
\end{proof}

For a set $A$, we let $K(A)$ denote the complete graph induced by $A$. In other words, $K(A)=(A,E)$ and \[E\df\set{(u,v)\in A\times A:u\not=v}.\]

For each ranking structure $R=(P,\tot)$ of a set $X$, we have two kinds of graph associated with it:
\begin{align*}
 \G{R} &= (X,E), \text{ where } E = \bigcup\set{X\times Y : X\not=Y \wedge X \tot Y}\\
 \GI{R} &= \G{R}\cup \bigcup_{A\in P} K(A)
\end{align*}

Intuitively, the graph $\G{R}$ is simply the transitive graph of the stratified order encoded by $R$. And the graph $\bigcup_{A\in P} K(A)$ is exactly the graph $\ICO{\G{R}}$, but in this case it is more intuitive to characterize it as the union of complete graphs.\\

Putting everything together we have the following characterization  of the class of all models of $\Tgso$.
\begin{definition} Let $\Cgso$ denote the class of all possible models for $\Tgso$. Then any model $\model{M}\in\Cgso$ is uniquely determined from 
\begin{itemize}
 \item the sets $E$, $\It{EO}$, and $O$
 \item the graphs $G_1$, $G_2$ and $G_3$
 \item a family $F$ of ranking structures on $\It{EO}$ indexed by the set $O$, i.e., $F=\set{R_o: o \in O}$, 
\end{itemize}
such that
\begin{enumerate}
 \item all conditions from \defref{class1} are satisfied
 \item all conditions from \defref{class2} are satisfied
 \item the graph $G_2$ is the intersection of all the graphs in the set $\set{\GI{R_o}: o \in O}$
 \item the graph $G_3$ is the intersection of all the graphs in the set $\set{\CO{\G{R_o}}: o \in O}$
 \item $\obefore^\model{M}=\set{(x,y,o) : (x,y) \text{ is an edge of the graph } \G{R_o}}$
 \item $\osim^\model{M}=\set{(x,y,o) : (x,y) \text{ is an edge of the graph } \bigcup_{A\in R_o} K(A)}$
\end{enumerate}\EOD 
\label{def:class3}
\end{definition}

\begin{theorem}[Satisfiability Theorem for $\Tgso$] If the class $\Cgso$ is defined as in \defref{class3}, then for any model $\model{M}\in\Cgso$, we have $\model{M}\models\Tgso$.
\label{theo:sat3}
\end{theorem}
\begin{proof} The fact that $\model{M}$ satisfies axioms \eref{e1} and \eref{e5} follows from the \theoref{sat1}. The fact that $\model{M}$ satisfies axioms \eref{gso1} and \eref{gso9} follows from the \theoref{sat2}.

Since each $R_o$ is a ranking structure on $EO$, from the way $\obefore^\model{M}$ and $\osim^\model{M}$ are defined, we know that $\model{M}$ satisfies axioms \eref{o1} and \eref{o2}.

Since $\osim^\model{M}$ is defined from the graphs $\bigcup_{A\in R_o} K(A)$ and each graph $\bigcup_{A\in R_o} K(A)$ is the incomparability graph of $\G{R_o}$, it follows that $\model{M}\models\eref{o5}$. Also since we construct the $\obefore^\model{M}$ relation from the graphs $\G{R_o}$ and each $\G{R_o}$ is a stratified order. Hence,  $\model{M}$ satisfies axioms \eref{o3}, \eref{o4} and \eref{o6} since these axioms are the conditions saying that $\lhd_o$ is a stratified order for every $o$ and we have $\G{R_o}=\lhd_o$. 

Recall axioms \eref{o7}-\eref{o10} together say that 
\begin{align}
(\forall o_1,o_2)\colq{\commutativity(o_1,o_2)\\\equiv(\forall o)(\obefore(o_1,o_2,o)\vee\obefore(o_2,o_1,o))} \label{eq:sat3.1} \\
(\forall o_1,o_2)\colq{\notlaterthan(o_1,o_2)\\\equiv(\forall o)\colqq{&\obefore(o_1,o_2,o)\\\vee&\osim(o_1,o_2,o)}} \label{eq:sat3.2} 
\end{align}
But this is equivalent to conditions (2) and (3) from \defref{class3}. \qed 
\end{proof}

\begin{theorem}[Axiomatizability Theorem for $\Tgso$] Any model  of $\Tgso$ is isomorphic to a structure of $\Cgso$.
\label{theo:axiom3} 
\end{theorem}
\begin{proof}
Let $\model{M}$ be a model of $\Tgso$. We will show that $\model{M}$ satisfies the conditions of the structures in $\Cgso$ from \defref{class2}. 

Since $\model{M}$ satisfies axioms \eref{e1}--\eref{e5}, from \theoref{axiom1} we can determine the sets $E=\event^\model{M}$ and the set $EO=\eventocc^\model{M}$, which satisfied the condition (1) of \defref{class3}.

Since $\model{M}$ satisfies axioms \eref{gso1}--\eref{gso9}, from \theoref{axiom2} we can determine the graphs $G_1$, $G_2$ and $G_3$ such that the condition (2) of \defref{class3} is satisfied.

Let $O=\eventocc^\model{M}$. Then since $\model{M}$ satisfies axioms \eref{o1}--\eref{o6}, we know that for all $o$ the induced relation $\lhd_o$ is a stratified order, so we can uniquely construct the family of ranking structure $F=\set{R_o: o \in O}$ from the set $\set{\lhd_o: o \in O}$. It is easy to check that the condition (5) and (6) of \defref{class3} are satisfied.

But since we already know that axioms \eref{o7}-\eref{o10} together are equivalent to conditions \eref{sat3.1} and \eref{sat3.2} from the proof of \theoref{sat3}, it follows that  $\model{M}$ satisfies  conditions (2) and (3) of \defref{class3}.
\qed
\end{proof}

\section{A semantic mapping to  PSL-core}
In this section, we will attempt to map a subset of $\Tgso$ to the PSL-core theory ($\Tpslcore$). We let $\Tgso^{-}$ to denote the theory consisting of axioms from \eref{gso1} to \eref{o10} and the following two axioms.

\begin{align}
(\forall x)\left(\eventocc(x)\vee\observation(x)\right) \label{eq:ex1}\\
(\forall x)\neg(\eventocc(x)\wedge\observation(x))\label{eq:ex2}
\end{align}

Axiom \eref{ex1} says that everything is either an \emph{event occurrence} or an \emph{observation}. And axiom \eref{ex2} says that the set of event occurrences and the set of observations are disjoint. 

The reason for considering the theory $\Tgso^{-}$ is that all of the interesting properties of $\Tgso$ concern with event occurrences and not with the events themselves. The second reason is that beside weakening the theory $\Tgso$, we do not see how we can establish a semantic mapping from $\Tgso$ to $\Tpslcore$ without introducing extra axioms into $\Tpslcore$.\\

\newcommand{\activity}{\mathsf{activity}}
\newcommand{\activityocc}{\mathsf{activity\_occurrence}}
\newcommand{\timepoint}{\mathsf{timepoint}}
\newcommand{\obj}{\mathsf{object}}
\newcommand{\occof}{\mathsf{occurrence\_of}}
\newcommand{\beginof}{\mathsf{beginof}}
\newcommand{\Endof}{\mathsf{endof}}
\newcommand{\before}{\mathsf{before}}
\newcommand{\participate}{\mathsf{participate\_in}}
\newcommand{\existsat}{\mathsf{exists\_at}}

To shorten our formulas, we need the following notation. For any formula $P(x)$ we define 
\[\bigl((\exists! x)\; P(x)\bigr) \equiv \Bigl(\bigl((\exists x)\; P(x)\bigr) \wedge \bigl((\forall x,y)\; P(x)\wedge P(y)\supset x=y\bigr)\Bigr)\]
In other words, we write $(\exists! x)\; P(x)$ to say that there exists a unique $x$ satisfying $P(x)$.

\begin{definition}[Interpretation of $\Tgso^-$ into $\Tpslcore$]We let $\pi$ denote the relative interpretation of the language of $\Tgso^-$ into $\Tpslcore$. Then the interpretation $\pi$ is defined as follows:
\allowdisplaybreaks
\begin{align*}
\pi_{\eventocc}(a) &\df \activity(a)\\
\pi_{\observation}(x) &\df 
	\colqq{ 
		&\obj(x)\wedge \bigl((\forall t)\, \existsat(x,t)\bigr)\\
	\wedge  &(\forall a)(\exists!o)\colqq{&\occof(o,a)\\ \wedge& (\exists! t)\, \participate(x,o,t)}
	}\\
\pi_{\obefore}(a_1,a_2,x)&\df
	\colqq{ 
		&\pi_{\observation}(x)\\
	\wedge  &(\exists o_1,o_2,t_1,t_2)
			\colqq{&\occof(o_1,a_1)\\ 
			\wedge& \occof(o_2,a_2)\\ 
			\wedge& \participate(x,o_1,t_1)\\ 
			\wedge& \participate(x,o_2,t_2)\\
			\wedge& \before(t_1,t_2)}
	}\\
\pi_{\osim}(a_1,a_2,x)&\df
	\colqq{ 
		&\pi_{\observation}(x)\\
	\wedge  &(\exists o_1,o_2,t_1,t_2)
			\colqq{&\occof(o_1,a_1)\\ 
			\wedge& \occof(o_2,a_2)\\ 
			\wedge& \participate(x,o_1,t_1)\\ 
			\wedge& \participate(x,o_2,t_2)\\
			\wedge& t_1=t_2}
	}\\
\pi_{\earlierthan}(a_1,a_2) &\df(\forall x)
	\colq{\pi_{\observation}(x)\supset \pi_{\obefore}(a_1,a_2,x)}\\
\pi_{\notlaterthan}(a_1,a_2) &\df(\forall x)
	\colqq{ &\pi_{\observation}(x)\\\supset&\colqq{ &\pi_{\obefore}(a_1,a_2,x)\\\wedge & \pi_{\osim}(a_1,a_2,x)}}\\
\pi_{\commutativity}(a_1,a_2) &\df(\forall x)
	\colqq{ &\pi_{\observation}(x)\\\supset&\colqq{ &\pi_{\obefore}(a_1,a_2,x)\\\vee & \pi_{\before}(a_2,a_1,x)}}
\end{align*}
\EOD
\label{def:rel}
\end{definition}

Intuitively, the interpretation means the following. If in $\Tgso^-$ each observation is a ``system run'', encoded by a stratified order of the event occurrences, which is observed by some implicit observer, then in $\Tpslcore$ we explicitly describe this observer as an object. For our interpretation, we are particularly interested in objects that participate in a unique activity occurrence of each activity at a unique time point. In other words, observers are objects satisfying the following properties: 
\begin{enumerate}
 \item The time point in which an object participates with an activity occurrence of an activity is exactly the time when the object observes the activity.
 \item The object observes every activity.
 \item The object only observes each activity exactly once.
\end{enumerate}

All of the other interpretations $\pi_{\earlierthan}$,  $\pi_{\notlaterthan}$ and $\pi_{\commutativity}$ can be easily determined from the observations that all observers observed.

\begin{theorem}
The interpretation $\pi$ defined in \defref{rel} is correct.  
\end{theorem}
\begin{proof}
It is easy to check that under the interpretation $\pi$, every axioms of $\Tgso^-$ is a theorem of $\Tpslcore$. Hence,  $\pi$ defined in \defref{rel} is a correct interpretation.
\qed 
\end{proof}

\section{Conclusion}
\sloppy
In this paper, we proposed in our knowledge the first version of a first-order theory for gso-structures in \cite{GJ,J0}. We avoid the difficulty of not being able to quantify over relations in first-order logic by introducing the relations $\obefore$ and $\osim$ which take an observation as one of their parameters. 

Using model-theoretic ontological techniques introduced in \cite{Gr2}, we  classified all possible models of $\Tgso$, where our key results are the satisfiability theorem and axiomatizability theorem for $\Tgso$. In our opinion, the classification of models of $\Tspec$, which decomposes the $\earlierthan$, $\notlaterthan$ and $\commutativity$ into smaller graphs, is especially insightful in understanding these three relations. Although the classification of observations using ranking structures is quite artificial, we could not figure out any simpler characterization.

We also give a very intuitive interpretation of the weaker theory $\Tgso^-$ into $\Tpslcore$, which shows that $\Tpslcore$ is strong enough to prove most of the theorems in $\Tgso$. The main philosophical difference between $\Tgso$ and $\Tpslcore$ is that causality relations are treated as logical relations without mentioning the concept of time in $\Tgso$ while the causality relations in $\Tpslcore$ are directly connected to timepoints of a reference timeline. 

The fact that   $\Tgso^-$ can be correctly interpreted inside $\Tpslcore$ also suggests that the soundness and completeness conditions might be too restrictive. One way to relax these conditions is to partition the observation set into ``legal'' and ``illegal'' observations, where legal observations are the ones satisfying the soundness and completeness conditions. This approach would also give us the ability to talk about illegal observations.

\end{document}